\documentclass[a4paper,10pt]{article}
\usepackage[utf8x]{inputenc}
\usepackage{amsmath}
\usepackage{amssymb}
\usepackage{amsthm}
\usepackage[dvips]{graphicx}
\usepackage{cite}
\usepackage{pgfplots}
\usepackage{geometry}
\usepackage[authoryear, round]{natbib}
\usepackage{amsfonts}
\providecommand{\U}[1]{\protect\rule{.1in}{.1in}}
\newtheorem{definition}{Definition}
\newtheorem{algorithm}{Algorithm}
\newtheorem{theorem}{Theorem}

\begin{document}

\title{Stratified Graphical Models - Context-Specific Independence in Graphical Models}
\author{Henrik Nyman$^{\ast,2}$, Johan Pensar$^{2}$, Timo Koski$^{3}$, Jukka Corander$^{1,2}$ \\
$^{1}$Department of Mathematics and statistics, University of Helsinki, Finland \\
$^{2}$Department of Mathematics, \AA bo Akademi University, Finland \\
$^{3}$Department of Mathematics \\ KTH Royal Institute of Technology, Stockholm, Sweden \\
$^{\ast}$Corresponding author, Email: hennyman@abo.fi}
\date{}
\maketitle

\begin{abstract}
Theory of graphical models has matured over more than three decades to provide the backbone for several classes of models that are used in a myriad of applications such as genetic mapping of diseases, credit risk evaluation, reliability and computer security, etc. Despite of their generic applicability and wide adoptance, the constraints imposed by undirected graphical models and Bayesian networks have also been recognized to be unnecessarily stringent under certain circumstances. This observation has led to the proposal of several generalizations that aim at more relaxed constraints by which the models can impose local or context-specific dependence structures. Here we consider an additional class of such models, termed as stratified graphical models. We develop a method for Bayesian learning of these models by deriving an analytical expression for the marginal likelihood of data under a specific subclass of decomposable stratified models. A non-reversible Markov chain Monte Carlo approach is further used to identify models that are highly supported by the posterior distribution over the model space. Our method is illustrated and compared with ordinary graphical models through application to several real and synthetic datasets.
\end{abstract}


\section{Introduction}
Along the path of development of the statistical theory of graphical models (GMs) largely set by the classic works of \citet{Darroch80} and \citet{Lauritzen89}, multifaceted generalizations of the original Markov dependence concepts have flourished as the field gained momentum. Despite of the versatility of graphical models to encode the dependence structure over a set of discrete variables, there are several alternative model classes that are motivated by the failure of GMs to capture some forms of dependence or independence. For instance, hierarchical log-linear models that lack a direct graphical model representations were considered extensively already before the theory of graphical models took a concrete form, see for instance \citet{Haberman74}, \citet{Bishop07}, and more recently \citet{Hara12}. A particular challenge related to such models is the burdensome interpretation, which is one of the core advantages of graphical models.

An observation independently made in several enhancements of the theory of graphical models for discrete multivariate distributions is that the use of the basic concept of conditional independence may casually hide the existence of multiple local or context-dependent independencies. Using the theory of log-linear models for contingency tables as their basis, \citet{Corander03a}, \citet{Eriksen99, Eriksen05}, and \citet{Hojsgaard03, Hojsgaard04} introduced a variety of ways to generalize graphical models. The common notion in these models is that the conditional independence is replaced by an independence that holds only in a subspace of the outcome space of the variables included in a particular condition. Such restrictions may for instance be imposed in a recursive fashion as in \citet{Hojsgaard04}, in which case a variable that has been included in a context to split contingency tables into subsets where distinct dependence structures are imposed, can no longer itself be a subject to a local independence statement conditional on other variables. Completely independently of these developments found in the statistical literature, the machine learning community has witnessed the development of context-dependent Bayesian networks in \citet{Boutilier96}, \citet{Friedman96}, and \citet{Koller09}. The recursive approach has been considered also in this setting, as \citet{Boutilier96} introduced trees of conditional probability tables which form the backbone of Bayesian networks.

The above cited methods for obtaining a context-specific dependence structure for a set of variables impose rather extensive restrictions. In order to simplify statistical inference about the model parameters and learning of the model structure, we here aim at loosening the restrictions by introducing a larger and more general class of \textit{stratified graphical models} (SGMs), expanding the results of \citet{Corander03a}. The notion of stratification referred to here is distinct from that used in \citet{Geiger01}, who considered stratified exponential families for graphical models with hidden variables. In our framework stratification refers instead solely to observed variables. SGMs offer the advantage that context-specific independencies can be read directly off the graphs, promoting the comprehension of the dependence structure. We consider Bayesian inference for the class of SGMs and show that marginal likelihoods can be calculated analytically for a subclass of decomposable models. Learning of model structures associated with high posterior probabilities is performed using the non-reversible Markov chain Monte Carlo (MCMC) algorithm introduced in \citet{Corander08}.

This paper is organized as follows. SGMs are formally introduced in Section \ref{secSGM}. An analytical expression for the marginal likelihood given a decomposable SGM is derived in Section \ref{secML}. In section \ref{secAlgorithm} we present an MCMC-based search algorithm which is used to discover models associated with high posterior probabilities. Several synthetic and real datasets are used to illustrate the potential of SGMs in Section \ref{secRes}. The final section provides some concluding remarks along with some ideas for future research related to these models.

\section{Stratified Graphical Models}
\label{secSGM}
To enable the presentation of SGMs, some of the central concepts from the theory of GMs are first introduced. For a comprehensive account of the statistical and computational theory of probabilistic GMs, see \citet{Whittaker90}, \citet{Lauritzen96}, and \citet{Koller09}. While the terms node and variable are closely related when considering graphical models, we will in this article strive to use the notation $X_{\delta}$ when referring to the variable associated to node $\delta$. Let $G(\Delta,E)$, be an undirected graph, consisting of a set of nodes $\Delta$ which represent a set of random variables and of a set of undirected edges $E\subseteq\{\Delta \times\Delta\}$.  It is assumed throughout this article that all considered variables are binary. However, the introduced theory can readily be extended to finite discrete variables. 

For a subset of nodes $A \subseteq \Delta$, $G_{A}=G(A,E_{A})$ is a subgraph
of $G$, such that the nodes in $G_{A}$ are equal to $A$ and the edge set
comprises those edges of the original graph for which both nodes are in $A$,
i.e. $E_{A} = \{A \times A\} \cap E$. The outcome space for the variables $X_A$, where $A \subseteq \Delta$, is
denoted by $\mathcal{X}_{A}$ and an element in this space by $x_{A}
\in \mathcal{X}_{A}$. Given our restriction to binary variables, the
cardinality $|\mathcal{X}_{A}|$ of $\mathcal{X}_{A}$ equals $2^{|A|}$. Two
nodes $\gamma$ and $\delta$ are \textit{adjacent} in a graph if $\{\gamma, \delta\}\in E$, 
that is an edge exists between them. A \textit{path} in a graph is a
sequence of nodes such that for each successive pair within the sequence the
nodes are adjacent. A \textit{cycle} is a path that starts and ends with the same node. 
A \textit{chord} in a cycle is an edge between two non-consecutive nodes in the cycle. 
Two sets of nodes $A$ and $B$ are said to be \textit{separated} by
a third set of nodes $S$ if every path between nodes in $A$ and nodes in $B$ contains
at least one node in $S$. A graph is defined as \textit{complete} when all pairs of nodes
in the graph are adjacent.

A graph is defined as \textit{decomposable} if all cycles found in the graph containing four or more unique nodes contains at least one chord. A \textit{clique} in a graph is a set of nodes $C$ such that the subgraph $G_{C}$ is complete and there exists no other set
$C^*$ such that $C \subset C^*$ and $G_{C^*}$ is also complete. The set of cliques in the graph $G$ will be denoted by $\mathcal{C}(G)$. The set of \textit{separators}, $\mathcal{S}(G)$, in the decomposable graph $G$ can be obtained through intersections of the cliques of $G$ ordered in terms of a junction tree, see e.g. \citet{Golumbic04}. A graphical model can be defined as the pair $G=G(\Delta,E)$ and the joint distribution $P_{\Delta}$ on the variables $X_{\Delta}$, such that $P_{\Delta}$ factorizes according to $G$ (see equation \eqref{factor} for decomposable graphs). Given only the graph of a GM it is possible to ascertain if two sets of random variables $X_A$ and $X_B$ are conditionally independent given another set of variables $X_S$, due to the global Markov property
\[
X_A \perp X_B \mid X_S,\text{ if } S \text{ separates } A \text{ from } B \text{ in } G.
\]
A statement of conditional independence of two variables $X_{\delta}$ and $X_{\gamma}$ given $X_S$ imposes fairly strong restrictions to the joint distribution since the condition $P(X_{\delta}, X_{\gamma} \mid X_S) = P(X_{\delta} \mid X_S) P(X_{\gamma} \mid X_S)$ must hold for any joint outcome of the variables $X_S$. The idea common to context-specific independence models is to lift some of those restrictions to achieve more flexibility in terms of model structure. Exactly which restrictions are allowed to be simultaneously lifted varies considerably over the proposed model classes.

Consider a GM with the complete graph spanning three nodes $(1, 2, 3)$, which specifies that there are no conditional independencies among the variables $X_1$, $X_2$, and $X_3$. However, if the probability $P(X_1=1, X_2=x_2, X_3=x_3)$ factorizes into the product $P(X_1=1) P(X_2=x_2 \mid X_1=1) P(X_3=x_3 \mid X_1=1)$ for all outcomes $x_2 \in \{0,1\}, x_3 \in \{0,1\}$, then a simplification of the joint distribution is hiding beneath the graph. This simplification can be included in the graph by labeling the edge $(2, 3)$ with the \textit{stratum} where the context-specific independence $X_2 \perp X_3 \mid X_1=1$ of the two variables holds, as illustrated in Figure \ref{SGMs}a.
\begin{figure}[htb]
\begin{center}
\includegraphics{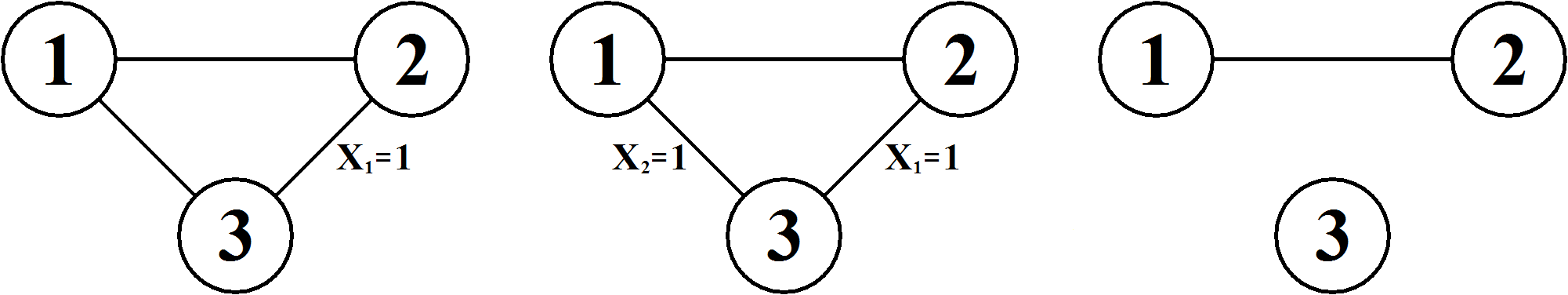}
\end{center}
\caption{Graphical representation of the dependence structures of three variables. In (a) the stratum $X_1=1$ is shown as a label on the edge $\{2, 3\}$, in (b) strata $X_1=1$ and $X_2=1$ are shown as labels on the edges $\{2, 3\}$ and $\{1, 3\}$, respectively, in (c) an undirected graph with the cliques $\{1, 2\}$ and $\{3\}$ is shown.}
\label{SGMs}
\end{figure}
The following is a formal definition of what is intended by a stratum.
\begin{definition} \label{stratum} Stratum. Let the pair $(G, P_{\Delta})$ be a GM for $\Delta$. For all
$\{\delta,\gamma\} \in E$, let $L_{\{\delta,\gamma\}}$ denote the set of nodes adjacent
to both $\delta$ and $\gamma$. For a non-empty $L_{\{\delta,\gamma\}}$, define the stratum
of the edge  $\{\delta,\gamma\}$ as the subset $\mathcal{L}_{\{\delta,\gamma\}}$ of outcomes
$x_{L_{\{\delta,\gamma\}}} \in \mathcal{X}_{L_{\{\delta,\gamma\}}}$ for which $X_{\delta}$ and $X_{\gamma}$ are
independent given $X_{L_{\{\delta,\gamma\}}} = x_{L_{\{\delta,\gamma\}}}$, i.e. $\mathcal{L}_{\{\delta,\gamma\}} = 
\{ x_{L_{\{\delta,\gamma\}}} \in  \mathcal{X}_{L_{\{\delta,\gamma\}}} : X_{\delta} \perp X_{\gamma} \mid X_{L_{\{\delta,\gamma\}}} =  x_{L_{\{\delta,\gamma\}}} \}$.
\end{definition}

A label on an edge in a graph is a graphical representation of a corresponding stratum. The idea of context-specific independence generalizes readily to a situation where multiple strata for distinct pairs of variables are considered. Figure \ref{SGMs}b displays the complete graph for three nodes with the edges $\{2, 3\}$ and $\{1, 3\}$ labeled with the strata $X_1=1$ and $X_2=1$, respectively. In addition to the context-specific independence statement present in Figure \ref{SGMs}a, here we have the simultaneous restriction that $X_1 \perp X_3 \mid X_2=1$, such that $P(X_1=x_1, X_2=1, X_3=x_3) = P(X_2=1) P(X_1=x_1 \mid X_2=1) P(X_3=x_3 \mid X_2=1)$ for all outcomes $x_1\in \{0,1\}, x_3\in\{0,1\}$. This pair of restrictions does not imply that $P(X_3=x_3) =  P(X_3=x_3 \mid X_1=1, X_2=1)$ as would be the case given the graph in Figure \ref{SGMs}c. It does, however, imply that the information contained about $X_3$ in the knowledge that $X_1=1$ and $X_2=1$ must be the same, i.e. $P(X_3=x_3 \mid X_1=1) = P(X_3=x_3 \mid X_2=1) = P(X_3=x_3 \mid X_1=1, X_2=1)$.

The following definition is a slight modification from \citet[p.~496]{Corander03a} and
formalizes an extension to ordinary graphical models. This class of models allow for simultaneous
context-specific independence to be represented using a set of strata, partitioning the outcome
space of the variables $X_{\Delta}$.
\begin{definition}
Stratified graphical model (SGM). A stratified graphical model is defined by the triple $(G, L, P_{\Delta})$, where $G$ is the underlying graph, $L$ equals the joint collection of all strata $\mathcal{L}_{\{\delta,\gamma\}}$ for the edges of $G$, and $P_{\Delta}$ is a joint distribution on $\Delta$ which factorizes according to the restrictions imposed by $G$ and $L$.
\end{definition}

The pair $(G, L)$ consisting of the graph $G$ with the labeled edges determined by $L$ will be referred to as a stratified graph (SG), usually denoted by $G_L$. When the collection of strata $L$ is empty, $G_{L}$ equals $G$. The distribution $P_{\Delta}$ is defined as \textit{faithful} if it does not contain any conditional or context-specific independencies that are not deducible from $(G, L)$. Figure \ref{SGMs}a illustrates an SG with $G$ equal to the complete graph and $L$ including the single stratum $\mathcal{L}_{\{2, 3\}} = \{X_1=1\}$. Correspondingly, the SG shown in Figure \ref{SGMs}b has the same underlying graph with the two strata $L=\{\mathcal{L}_{\{2, 3\}} = \{X_1=1\}, \mathcal{L}_{\{1, 3\}}=\{X_2=1\}\}$.

The remainder of this section will be used to determine a framework that will allow for the derivation of an analytical expression of the marginal likelihood of a dataset given a stratified graph. Unfortunately, this will involve introducing a set of restrictions to the graph structure, resulting in a subclass of \textit{decomposable SGMs}. The restrictions imposed here are, however, far less extensive then those imposed in \citet{Corander03a}. In addition, \citet{Corander03a} did not consider structural learning of the context-specific graphs by using posterior probabilities, instead a simpler approach with penalized predictive entropies was adopted for such inference.

Consider a stratified graph with a decomposable underlying graph $G$ having the cliques $\mathcal{C}(G)$ and separators $\mathcal{S}(G)$. The SG is defined as decomposable if no labels are assigned to edges in any separator and in every clique all labeled edges have at least one node in common.
\begin{definition}
Decomposable SG. Let $(G, L)$ constitute an SG with $G$ being decomposable. Further, let $E_{L}$ denote the set of all labeled edges, $E_{C}$ the set of all edges in clique $C$, and $E_{\mathcal{S}}$ the set of all edges in the separators of $G$. The SG is defined as decomposable if
\[
E_{L}\cap E_{\mathcal{S}}=\emptyset,
\]
and
\[
E_{L} \cap E_{C} = \emptyset \hspace{0.4cm} \text{ or } \hspace{0.0cm} \bigcap_{\{\delta,\gamma\}\in E_{L}\cap E_{C}} \hspace{-0.5cm} \{\delta,\gamma\} \hspace{0.1cm} \neq \hspace{0.1cm} \emptyset \hspace{0.1cm} \text{ for all } \hspace{0.1cm} C \in \mathcal{C}(G).
\]
\end{definition}
An SGM where $(G, L)$ constitutes a decomposable SG is termed a decomposable SGM. The graphs depicted in Figures \ref{SGMs}a and \ref{SGMs}b are examples of decomposable SGs. Figure \ref{order} displays three SGs with identical underlying graphs, where it is assumed that the nodes are ordered topologically. This entails that it is not necessary to include the variables in the graphical representation of a stratum. Instead of writing a label as $(X_1=0, X_2=0)$, it is sufficient to write $(0, 0)$, as it is uniquely determined which nodes are adjacent to both nodes in the labeled edge. The SG in Figure \ref{order}a is decomposable, the SGs in Figures \ref{order}b and \ref{order}c are not. The graph in Figure \ref{order}b is not decomposable since the clique $\{1, 2, 3, 4\}$ contains the labeled edges $\{1, 2\}$ and $\{3, 4\}$ which have no nodes in common. The graph in Figure \ref{order}c contains the labeled edge $\{1, 4\}$ which also constitutes the separator of cliques $\{1, 2, 3, 4\}$ and $\{1, 4, 5\}$.
\begin{figure}[htb]
\begin{center}
\includegraphics{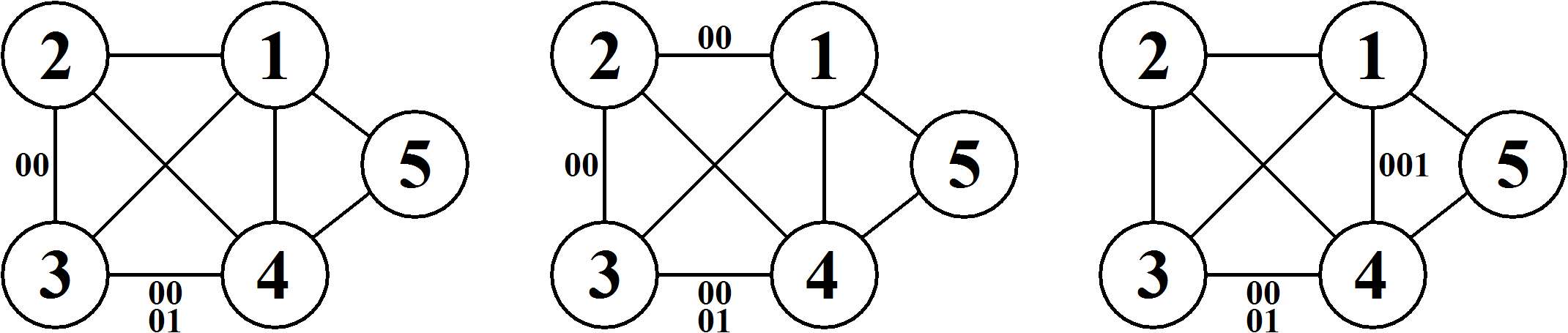}
\end{center}
\caption{Three SGs with the same underlying graph. The graph in (a) is decomposable, the graphs in (b) and (c) are not.}
\label{order}
\end{figure}

As shown in the next section, for a decomposable stratified graph it is possible to calculate the marginal likelihood of a dataset analytically using a modification of the procedure introduced by \citet{Cooper92}. This is due to the fact that a joint distribution faithful to a decomposable SG can be factorized using a \textit{minimal factorization}.
\begin{definition}
Minimal factorization. Given an ordering of the variables $X_{\Delta}$, let $\prod_{\delta \in \Delta}$ $P(X_{\delta} \mid X_{A(\delta)})$ be a factorization of the joint distribution $P_{\Delta}$. Each set $A(\delta)$ is a subset of $B_{\delta}$, where $B_{\delta}$ denotes the complete set of nodes for which all variables in $X_{B_{\delta}}$ precede $X_{\delta}$ in the ordering. The set $B_{\delta}$ is defined as empty if $X_{\delta}$ is the first variable in the ordering. If a factor $P(X_{\delta} \mid X_{A(\delta)})$ is such that there exists a non-empty subset $D \subseteq A(\delta)$ for which $X_{\delta} \perp X_D \mid X_{\Delta} \backslash (X_D \cup X_{\delta})$, the factor contains a false dependency. A factorization that contains no false dependencies is defined as a minimal factorization.
\end{definition}
If we for instance look at the graph in Figure \ref{SGMs}c we can deduce that $P(X_1, X_2, X_3)$ can be factorized both as $P(X_1) P(X_2 \mid X_1) P(X_3 \mid X_1, X_2)$ and $P(X_1) P(X_2 \mid X_1) P(X_3)$. However, the first factorization does not constitute a minimal factorization since $X_3 \perp (X_1, X_2)$.

\begin{theorem}
\label{minimal} A joint distribution, faithful to a stratified graph, possesses a minimal factorization if and only if it is faithful to a decomposable stratified graph.
\end{theorem}
\begin{proof}
See Appendix A.
\end{proof}

\section{Calculating the Marginal Likelihood for Decomposable Stratified Graphs}
\label{secML}
Let $\mathbf{X}$ denote a data matrix of $n$ binary vectors, each containing $d=|\Delta|$ elements. We use $X_{A}$ to denote the set of variables $\{X_{\delta}: \delta \in A\}$ and correspondingly $\mathbf{X}_{A}$ to denote the subset of $\mathbf{X}$ for the variables in $A$. A probability distribution over the outcome space $\mathcal{X}_{A}$, is determined by a parameter vector $\theta \in \Theta$ where every element $\theta_{i}$ specifies the probability of a specific outcome $x_{A}^{(i)} \in\mathcal{X}_{A}$. The number of such outcomes will subsequently be denoted by $k$. Bayesian inference about undirected graphs and stratified graphs is derived using the posterior distribution over the model space. Given a prior distribution $P(G)$ (or $P(G_L)$), over a model space, the posterior equals $P(G \mid \textbf{X}) = P(\textbf{X} \mid G)P(G) / \sum_{G \in \mathcal{G}} P(\textbf{X} \mid G)P(G)$, where $P(\textbf{X} \mid G)$ is the marginal likelihood of the data given a graph, and $\mathcal{G}$ is the model space.

For an arbitrary decomposable graph $G$, the joint distribution of $\Delta$ factorizes as
\begin{equation}
\label{factor}
P_{\Delta}(X_{\Delta})=\frac{\prod_{C \in \mathcal{C}(G)} P_{C} (X_{C})}{\prod_{S \in \mathcal{S}(G)}P_{S}(X_{S})},
\end{equation}
where $\mathcal{C}(G)$ and $\mathcal{S}(G)$ are the cliques and separators,
respectively, of $G$. Using a prior distribution for the model parameters that
also enjoys the Markov properties with respect to $G$, the marginal likelihood
of the data $\textbf{X}$ factorizes accordingly \citep{Dawid93}
\begin{equation}
P(\mathbf{X} \mid G)=\frac{\prod_{C\in\mathcal{C}(G)}P_{C}(\mathbf{X}_{C})}
{\prod_{S\in\mathcal{S}(G)}P_{S}(\mathbf{X}_{S})},
\label{clisep}
\end{equation}
where for any subset $A \subseteq \Delta$ of nodes $P_{A}(\mathbf{X}_{A})$ denotes the corresponding marginal likelihood of the
subset $\mathbf{X}_{A}$ of data. By a suitable choice of prior distribution,
these marginal likelihoods can be calculated analytically as follows. Let
$n_{A}^{(i)}$ be the number of occurrences of the outcome $x_{A}^{(i)}$ in the
dataset $\mathbf{X}_{A}$ and let the probabilities determining the
corresponding distribution have the Dirichlet $(\alpha_{A_{1}}
,\ldots,\alpha_{A_{k}})$ distribution as the prior. Then, the
marginal likelihood of $\mathbf{X}_{A}$ equals
\[
P_{A}(\mathbf{X}_{A})=\int_{\Theta}\prod_{i=1}^{k}(\theta_{i})^{n_{A}^{(i)}
}\cdot\pi_{A}(\theta)d\theta,
\]
where $\pi_{A}(\theta)$ is the density function of the Dirichlet prior
distribution. By the standard properties of the Dirichlet integral, the marginal likelihood can be further written as
\begin{equation}
P_{A}(\mathbf{X}_{A})=\frac{\Gamma(\alpha)}{\Gamma(n+\alpha)}\prod_{i=1}
^{k}\frac{\Gamma(n_{A}^{(i)}+\alpha_{A_{i}})}{\Gamma(\alpha_{A_{i}}
)},
\label{clique}
\end{equation}
where $\Gamma$ denotes the gamma function and
\[
\alpha=\sum_{i=1}^{k}\alpha_{A_{i}}.
\]
The above result can be utilized as the basis when developing a corresponding expression for decomposable SGs. For these graphs each clique and separator can be considered separately, and the factorization in \eqref{factor} used. This is due to the fact that for a decomposable SG the nodes in a labeled edge $\{\delta ,\gamma\}$ and the nodes in $L_{\{\delta,\gamma\}}$ all belong to the same clique, as labels may not be placed on separators. Hence, a label on an edge in one clique cannot imply changes to the dependence structure between variables associated with any other clique. Given a clique and its associated labels defined in $G_{L}$, it is necessary to define a factorization of the distribution $P_{C}(X_{C})$ using a sequence of conditional distributions. To achieve this we introduce, in accordance with the proof of Theorem \ref{minimal}, a particular ordering of the variables in the clique such that the variable corresponding to the node which all labeled edges have in common is last in the ordering. In the case where all labeled edges have two nodes in common, the last variable in the ordering can be chosen arbitrarily between them, and in the case with no labeled edges the ordering of the variables is arbitrary. When the factorization is based on such an ordering, it becomes clear which dependencies can be excluded from the last conditional distribution and it can be guaranteed that no false dependencies are employed.

An alternative way of formulating the factorization process for a clique 
is to consider the variables that precede the variable $X_\delta$ in
the ordering as parents of $X_{\delta}$, denoted by $X_{\Pi_{\delta}}$.
Hence, except for the last variable in the ordering, all variables depend in
their conditional distribution on each of the values of their parents. For the
last variable some outcomes of its parents will have the same effect and these
values can consequently be grouped together, as is done using default tables or CPT-trees
by \citet{Friedman96} and \citet{Boutilier96}. As an example, consider the SG in Figure \ref{SGMs}a, where the
parent outcomes $(X_1=1, X_2=0)$ and $(X_1=1, X_2=1)$ for $X_3$ are grouped together.
Correspondingly, for the SG in Figure \ref{SGMs}b, the parent outcomes
$(X_1=1, X_2=0)$, $(X_1=1, X_2=1)$, and $(X_1=0, X_2=1)$ are grouped together. This means
that there are effectively only two \textit{distinguishable} parent combinations for
$X_3$, comprising of $\{(X_1=1, X_2=0), (X_1=1, X_2=1), (X_1=0, X_2=1)\}$ and $\{(X_1=0, X_2=0)\}$. 

Using the ordering of variables discussed above the marginal likelihood
$P_{C}(\mathbf{X}_{C})$ for a clique of a decomposable SG can be calculated
using a modified version of the formula introduced by \citet{Cooper92} for the marginal likelihood of a Bayesian network. Our modification is
defined as
\begin{equation}
P_{C}(\mathbf{X}_{C})=\prod_{j=1}^{d}\prod_{l=1}^{q_{j}}\frac{\Gamma
(\sum_{i=1}^{k_{j}}\alpha_{jil})}{\Gamma(n(\pi_{j}^{l})+\sum_{i=1}^{k_{j}
}\alpha_{jil})}\prod_{i=1}^{k_{j}}\frac{\Gamma(n(x_{j}^{i} \mid \pi_{j}^{l}
)+\alpha_{jil})}{\Gamma(\alpha_{jil})},
\label{CHmodi}
\end{equation}
where $d$ equals the number of variables in the clique $C$, $q_{j}$ is the
number of \textit{distinguishable} parent combinations for variable $X_j$
(i.e. there are $q_j$ distinct conditional distributions for variable $X_j$), $k_{j}$ is the
number of possible outcomes for variable $X_j$, $\alpha_{jil}$ is the
hyperparameter corresponding to the outcome $i$ of variable $X_j$ given that the
parental combination of $X_j$ equals $l$, $n(\pi_{j}^{l})$ is the number of
observations of the combination $l$ for the parents of variable $X_j$, and
finally, $n(x_{j}^{i} \mid \pi_{j}^{l})$ is the number of observations where the
outcome of variable $X_j$ is $i$ given that the observed outcome of the parents
of $X_j$ equals $l$. Note that in this context a parent configuration $l$ is not necessarily comprised of a
single outcome of the parents of variable $X_j$, but rather a \textit{group} of outcomes with an
equivalent effect on $X_j$.

The hyperparameters of the Dirichlet distribution are determined by imposing the following two requirements:

\begin{enumerate}
\item The resulting value of $P_{C}(\mathbf{X}_{C})$ is independent of the
initial ordering of the clique variables.

\item In the absence of labels the value of $P_{C}(\mathbf{X}_{C})$ is
identical to that given by \eqref{clique} when the hyperparameters in the
corresponding prior distribution are set equal to 1.
\end{enumerate}
These two requirements can be satisfied by the following choice of hyperparameters,
\[
\alpha_{jil} = \alpha_{jl} = \frac{k \cdot \lambda_{jl}}{\pi_{j} \cdot k_{j}},
\]
where $k$ again equals $|\mathcal{X}_{C}|$, $\pi_{j}$ is the total number of possible outcomes for the parents of variable $X_j$ and $k_{j}$ is the number of possible outcomes for variable $X_j$. Further, $\lambda_{jl}$ equals the number of outcomes for the parents of variable $X_j$ in \textit{group} $l$ with an equivalent effect on $X_j$, if $X_j$ is the last variable in the ordering. Otherwise, $\lambda_{jl}$ equals one. The values $P_{S}(X_{S})$, or even $P_{C}(X_{C})$ if $C$ is a clique containing no labeled edges, can now be calculated using either \eqref{clique} or \eqref{CHmodi} as these will yield the same results.

An essential additional element of learning decomposable SGs given a dataset is to ensure model identifiability. There may exist several different decomposable SGs that all induce the same dependence structure. As an example we can have a look at the SG in Figure \ref{order}a. Here, adding the label $(0, 1)$ to the edge $\{2, 3\}$ would merge the parent outcomes $(X_1=0, X_2=0, X_4=1)$ and $(X_1=0, X_2=1, X_4=1)$ of $X_3$. However, these two outcomes have already indirectly been merged by the existing labels. The label $(0, 0)$ on the edge $\{2, 3\}$ merges parent outcomes $(X_1=0, X_2=0, X_4=0)$ and $(X_1=0, X_2=1, X_4=0)$. The label $(0, 0)$ on the edge $\{3, 4\}$ merges parent outcomes $(X_1=0, X_2=0, X_4=0)$ and $(X_1=0, X_2=0, X_4=1)$. And the label $(0, 1)$ on the edge $\{3, 4\}$ merges parent outcomes $(X_1=0, X_2=1, X_4=0)$ and $(X_1=0, X_2=1, X_4=1)$. Meaning that the outcomes $\{(X_1=0, X_2=0, X_4=0), (X_1=0, X_2=0, X_4=1), (X_1=0, X_2=1, X_4=0), (X_1=0, X_2=1, X_4=1)\}$ already form a group with all outcomes having identical affect on $X_3$, thus the inclusion of the label $(0, 1)$ to the edge $\{2, 3\}$ does not alter the dependence structure. To exclude this possibility, we introduce the concept of \textit{maximal regular} SGMs.
\begin{definition}
Maximal regular SGM. A decomposable SGM is defined as maximal regular if no elements may be added to $L$ without altering the dependence structure. Further, the stratum associated with each edge $\{\delta,\gamma\}$ must be a proper subset of $\mathcal{X}_{L_{\{\delta, \gamma\}}}$.
\end{definition}
The SG of a maximal regular SGM is termed a maximal regular SG. The regularity refers to the condition that an edge cannot be excluded completely from the graph as in an ordinary GM, by setting $\mathcal{L}_{\{\delta,\gamma\}} = \mathcal{X}_{L_{\{\delta,\gamma\}}}$. Maximal regular SGMs constitute a subset of maximal SGMs. In contrast to the class of all SGMs, maximal regular SGMs always entail different dependence structures.
\begin{theorem}
The dependence structure induced by two maximal regular SGMs with stratified graphs $G_L^1$ and $G_L^2$, respectively, are identical if and only if $G_L^1 = G_L^2$.
\end{theorem}
\textit{Proof}. Let $G_L^1$ and $G_L^2$ denote two maximal regular SGMs with identical dependence structure. Assume first that the two underlying graphs, $G_1$ and $G_2$, differ in that a single edge $\{\delta, \gamma\}$ is present in $G_1$ but absent in $G_2$. Examining $G_2$ it is clear that $X_{\delta} \perp X_{\gamma} \mid X_{\Delta} \backslash (X_{\delta} \cup X_{\gamma})$. From this we can conclude that for $G_L^1$, utilizing the additional information that $G_1$ is a decomposable graph, given any outcome of the variables in $L_{\{\delta, \gamma\}}$ the variables $X_{\delta}$ and $X_{\gamma}$ are independent. This contradicts the assumption that $\mathcal{L}_{\{\delta,\gamma\}} \neq \mathcal{X}_{L_{\{\delta,\gamma\}}}$ in $G_L^1$ and consequently the two graphs $G_1$ and $G_2$ must be identical. Assume now that the set of strata $L_1$ in $G_L^1$ contains an element which is not present in the set of strata $L_2$ in $G_L^2$. This would imply that the element could be added to $L_2$ without altering the dependence structure of $G_L^2$, leading to a contradiction as $G_L^2$ cannot be a maximal regular SGM. $\square$

When learning SGs from data, we will restrict the attention to the class of maximal regular SGs. This means that we can avoid confusion over models having different appearances while leading to the same marginal likelihood due to their identical dependence structure.

\section{Algorithms for Bayesian Learning of SGs}
\label{secAlgorithm}
Bayesian learning of graphical models has attained a considerable interest, both in the statistical and computer science literature, see, e.g. \citet{Madigan94}, \citet{Dellaportas99}, \citet{Giudici99}, \citet{Corander03b}, \citet{Giudici03}, \citet{Koivisto04}, and \citet{Corander08}. Our learning algorithms described below belong to the class of non-reversible Metropolis-Hastings algorithms, introduced by \citet{Corander06} and later further generalized and applied to learning of graphical models in \citet{Corander08}. 

Let $\mathcal{M}$ denote the finite space of states over which the aim is to approximate the posterior distribution. In this paper we will run two separate types of searches. In one search the state space $\mathcal{M}$ will consist of all possible sets of labels, satisfying the restrictions of a maximal regular SG, for a given clique. In the second search the state space will be the set of decomposable undirected graphs combined with the optimal set of labels for that graph. For $M\in\mathcal{M}$, let $Q(\cdot \mid M)$ denote the proposal function used to generate a new candidate state given the current state $M$. Under the generic conditions stated in \citet{Corander08}, the probability with which any particular candidate is picked by $Q(\cdot \mid M)$ need not be explicitly calculated or known, as long as it remains unchanged over all the iterations and the resulting chain satisfies the condition that all states can be reached from any other state in a finite number of steps. To begin with a starting state, $M_{0}$, is determined. At iteration $t=1,2,...$ of the non-reversible algorithm, $Q(\cdot \mid M_{t-1})$ is used to generate a candidate state $M^{\ast}$, which is accepted with the probability
\begin{equation}
\min\left(  1,\frac{P(M^{\ast})P(\mathbf{X} \mid M^{\ast})}{P(M_{t-1}
)P(\mathbf{X} \mid M_{t-1})}\right),
\label{accept}
\end{equation}
where $P(M)$ is a prior probability assigned to $M$. The term $P(\mathbf{X} \mid M)$ denotes the marginal likelihood of the dataset $\mathbf{X}$ given $M$. If $M^{\ast}$ is accepted, we set $M_{t}=M^{\ast}$, otherwise we set $M_{t}=M_{t-1}$.

For non-reversible Markov chains the posterior probability $P(M \mid \mathbf{X})$ does not equal the stationary distribution. Instead, a consistent approximation of $P(M \mid \mathbf{X})$ is obtained by considering the space of distinct states $\mathcal{M}_{t}$ visited by time $t$ such that
\[
\hat{P}(M \mid \mathbf{X})=\frac{P(\mathbf{X} \mid M)P(M)}{\sum_{M\in\mathcal{M}_{t}
}P(\mathbf{X} \mid M)P(M)}.
\]
\citet{Corander08} did prove under rather weak conditions that this estimator is consistent, i.e.
\[
\hat{P}(M \mid \mathbf{X})\overset{a.s.}{\rightarrow}P(M \mid \mathbf{X}),
\]
as $t\rightarrow\infty$. As our main interest will lie in finding the posterior optimal state, i.e.
\[
\arg\mathop{\max}_{M\in\mathcal{M}}P(M \mid \mathbf{X}).
\]
it will suffice to identify
\[
\arg\mathop{\max}_{M\in\mathcal{M}}P(\mathbf{X} \mid M)P(M).
\]

The main goal of our search algorithm is to identify the stratified graph $G_L^{\text{opt}}$ optimizing $P(\mathbf{X} \mid G_L) P(G_L)$. The search is broken down into two parts. Under the assumption that the optimal set of labels is known for each underlying graph a Markov chain traversing the set of possible underlying graphs will eventually identify $G_L^{\text{opt}}$. Another search may be used in order to identify the optimal set of labels given the underlying graph. It was earlier concluded that the marginal likelihood for a decomposable SG can be factorized according to \eqref{clisep}. Due to this the search for the optimal set of labels can be conducted clique-wise. 

Given a decomposable underlying graph $G$ with the set of cliques $\mathcal{C}(G)$, a search is conducted to find the optimal set of labels for each clique $C \in \mathcal{C}(G)$. The sets of labels are assigned uniform priors and cancel each other out in the acceptance probability \eqref{accept}. Using the proposal function defined in Algorithm \ref{AlgoClique}, running a sufficient amount of iterations, we can be assured to find the optimal set of labels for each clique. Combining the sets of labels for each clique will result in an optimal labeling of the underlying graph. 
\begin{algorithm}
\label{AlgoClique} Proposal function used to find optimal set of labels for a clique $C \in \mathcal{C}(G)$.
\end{algorithm}
\noindent The starting state is defined as the empty set containing no labels. Let $L$ denote the current set of labels, and $LP$ the set of labels that can be added to $L$ without violating the restrictions of decomposable stratified graphs.
\begin{enumerate}
\item Set the candidate state $L^* = L$. 
\item If $LP$ is empty and $L$ is non-empty, delete a randomly chosen label in $L^*$.
\item If $L$ is empty and $LP$ is non-empty, add a randomly chosen label from $LP$ to $L^*$.
\item If both $L$ and $LP$ are non-empty, with probability $0.5$ delete a randomly chosen label in $L^*$, otherwise add a randomly chosen label from $LP$ to $L^*$.
\item Expand $L^*$ to include all labels that do not alter the dependence structure.
\item If $L^*$ satisfies the maximal regular restrictions set it as the candidate state, otherwise repeat steps 1-6.
\end{enumerate}

Using this procedure we can assume that the optimal labeling can be found for any underlying graph and we can move on to the search for the best underlying graph with optimal labeling. In this search, instead of using a uniform prior, we use a prior that penalizes dense graphs
\[
P(G_{L})\propto2^{d-f},
\]
where $d$ is the number of nodes in the underlying graph $G$ and $f$ is the number of free parameters in a distribution $P_{\Delta}$ faithful to $G$. This choice of prior is motivated by the fact that adding a label to a sparse graph often induces a context-specific independence in a larger stratum than adding a label to a dense graph. The value $2^{f - d}$ is a numerically convenient approximation of the number of unique dependence structures that can be derived by adding labels to an undirected graph. By looking at the conditional distributions for a variable $X_{\delta}$ with parents $X_{\Pi_{\delta}}$ in $G$, one can see that each parent outcome can be merged with a set of other outcomes by adding a label, removing a free parameter from $P_{\Delta}$ in the process. By adding different labels all but $d$ of the original $f$ free parameters in $P_{\Delta}$ can be removed, resulting in $2^{f - d}$ different dependence structures. This is, however, just an approximation as it is not possible to simultaneously remove any subset of the $f - d$ parameters by including labels. Using the proposal function in Algorithm \ref{AlgoSGM} we conduct the search for the best underlying graph with optimal labeling.
\begin{algorithm}
\label{AlgoSGM} Proposal function used to find the best underlying graph with optimal labeling.
\end{algorithm}
\noindent The starting state is set to be the graph containing no edges. Let $G$ denote the current graph with $G_L = (G, L)$ being the stratified graph with underlying graph $G$ and optimal labeling $L$.
\begin{enumerate}
\item Set the candidate state $G^* = G$. 
\item Randomly choose a pair of nodes $\delta$ and $\gamma$. If the edge $\{\delta, \gamma\}$ is present in $G^*$ remove it, otherwise add the edge $\{\delta, \gamma\}$ to $G^*$.
\item While $G^*$ is not decomposable repeat step 2.
\end{enumerate}
\noindent The resulting candidate state $G^*$ is used along with the corresponding optimal set of labels $L^*$ to form the stratified graph $G^*_L = (G^*, L^*)$ which is used when calculating the acceptance probability according to \eqref{accept}.

In the next section we will use the search operator defined here on a set of synthetic datasets in order to illustrate their efficiency. We will also apply the search operator to a couple of real datasets that have been subject of extensive study in the past. As we will see our results strongly support the use of models that enable the inclusion of context-specific independencies.

\section{Illustration of SG Learning from Data}
\label{secRes}
An SG including seven nodes is shown in Figure \ref{synthetic}a. A probability distribution following the dependence structure defined by this SG is used to generate several sets of data of varying size, this distribution is available in the form of Matlab code in Appendix B.
\begin{figure}[h]
\begin{center}
\includegraphics{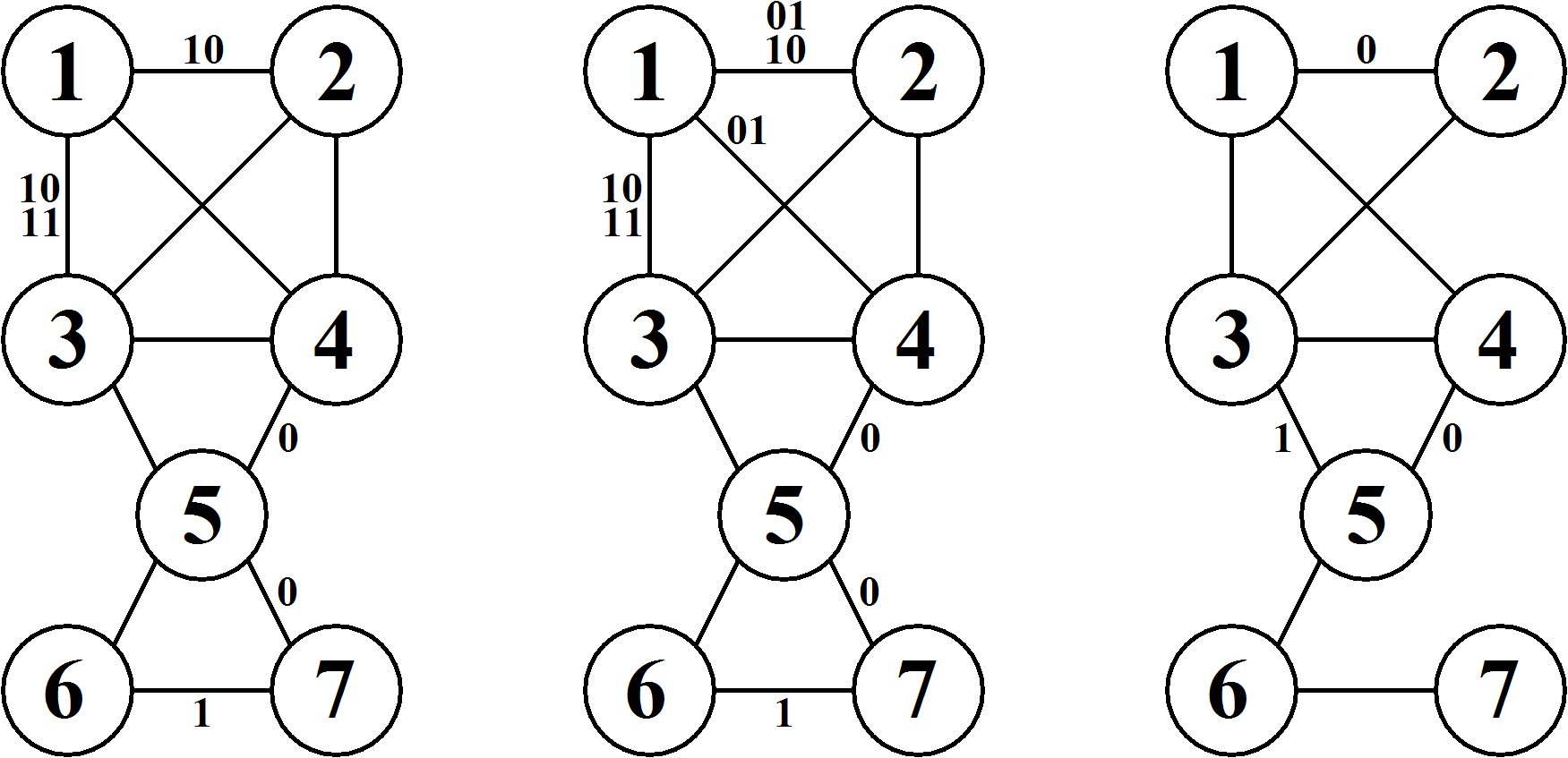}
\end{center}
\caption{Dependence structures of distributions used to determine required size of datasets.}
\label{synthetic}
\end{figure}
When the size of the dataset exceeds 1,000 observations the graphs with the highest posterior probability found by the search method defined in the previous section usually coincide with the generating model. However, as the number of observations drops the optimal graphs start to deviate from the generating graph. The SG in Figure \ref{synthetic}b is representative of the optimal graphs found for datasets including 500 observations. We can see that the underlying graph is still the correct one, but a number of extra labels have been added to the clique $\{1, 2, 3, 4\}$. The SG in Figure \ref{synthetic}c is representative of the optimal graphs for datasets containing only 100 observations. Here we can see that not only do the labels differ strongly from those of the generating graph but also that the underlying graph is missing a couple of edges.

The experiments based on synthetic data confirm that the search algorithms are performing as expected when the data generating structure is known. However, for datasets consisting of less than 1,000 observations the observed posterior modes usually differ from the generating model. This gives rise to an important question, namely, when trying to learn the model structure what is the required size of the dataset? For decomposable SGs, as we try to determine which labels to include in each clique, the size of the cliques will be of relevance. The larger the clique, the larger the set of possible labels, implying that more data is needed to have high probability of discovering the underlying dependence structure in a faithful manner. In an effort to determine the required number of observations in the dataset for cliques with three, four, and five variables we generate multiple datasets of varying size following the dependence structure induced by the graphs in Figure \ref{dataSizeSGM}.
\begin{figure}[h]
\begin{center}
\includegraphics{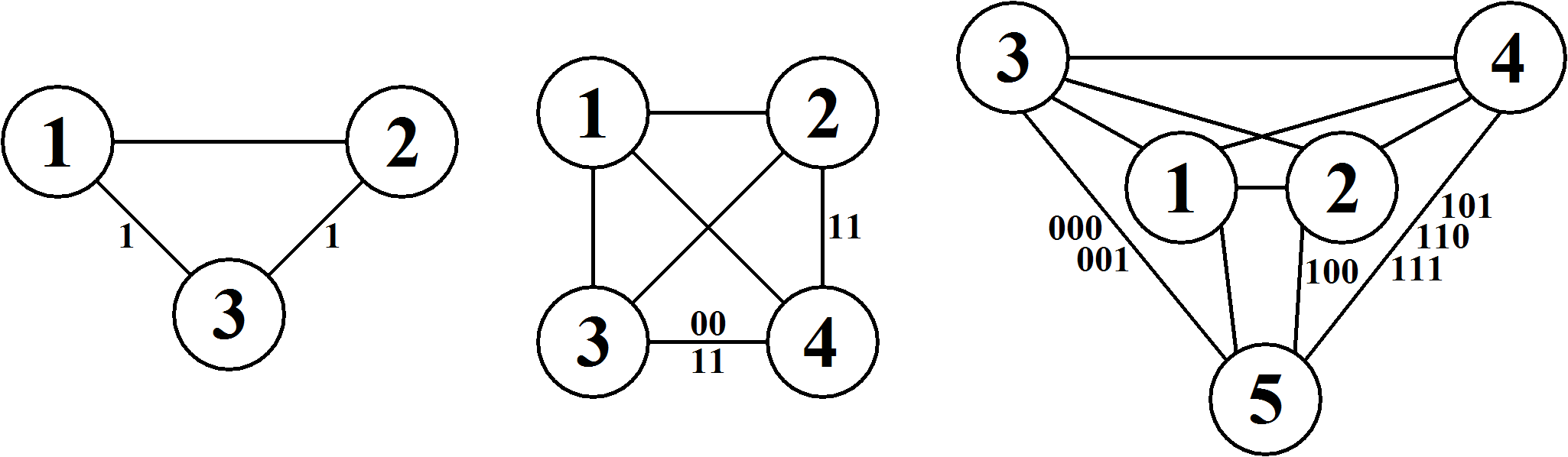}
\end{center}
\caption{Dependence structures of distributions used to determine required size of datasets.}
\label{dataSizeSGM}
\end{figure}
For each dataset size 1,000 datasets are generated, the method defined in Section \ref{secAlgorithm} is then used on each dataset to determine the optimal SG. The log marginal likelihoods for the datasets $\textbf{X}_i$ are calculated given the generating SG and the optimal SG, and denoted by $LML_i^{\text{gen}}$ and $LML_i^{\text{opt}}$, respectively. After which we calculate
\[
y = \exp(\frac{\frac{1}{1000} \sum_{i=1}^{1000} LML_i^{\text{gen}}}{n}) - \exp(\frac{\frac{1}{1000} \sum_{i=1}^{1000} LML_i^{\text{opt}}}{n}),
\]
where $n$ is the data size. The value $y$, which has the intuitive interpretation of being the average difference between the probabilities of observing a single outcome in the dataset given the two models, is used as the measure when determining if the size of the dataset is large enough. Obviously, $y$ will only take negative values, but as the data size increases the optimal model will tend towards the generating model resulting in $y$ tending to zero. For any other model than the generating model $y$ would not tend towards zero.

Figure \ref{dataSize} contains values of $y$ plotted against different data sizes for clique sizes three, four, and five. The three curves all converge to zero but the size of data required for convergence does vary. When the clique consists of only three variables it, in this case, appears that 150 observations are enough for the curve to level off. For four variables we need about 500 observations. The curve corresponding to the case with five variables is more difficult to interpret, it continues to rise at almost the same angle after about 700 observations. In this case it is likely that more than 1,500 observations will be necessary in order to ascertain the generating model. It is very difficult if not impossible, just as in the case for traditional graphical models, to give a general rule of how large the datasets need to be. Not only does the dependence structure, but also the specific probability distribution, affect the complexity of the problem. Some dependencies will be weaker and will require more data in order to be identified. However, the results above give an indication of the required data size and how rapidly this size increases with the size of the cliques.
\begin{figure}[h]
\begin{center}
\includegraphics{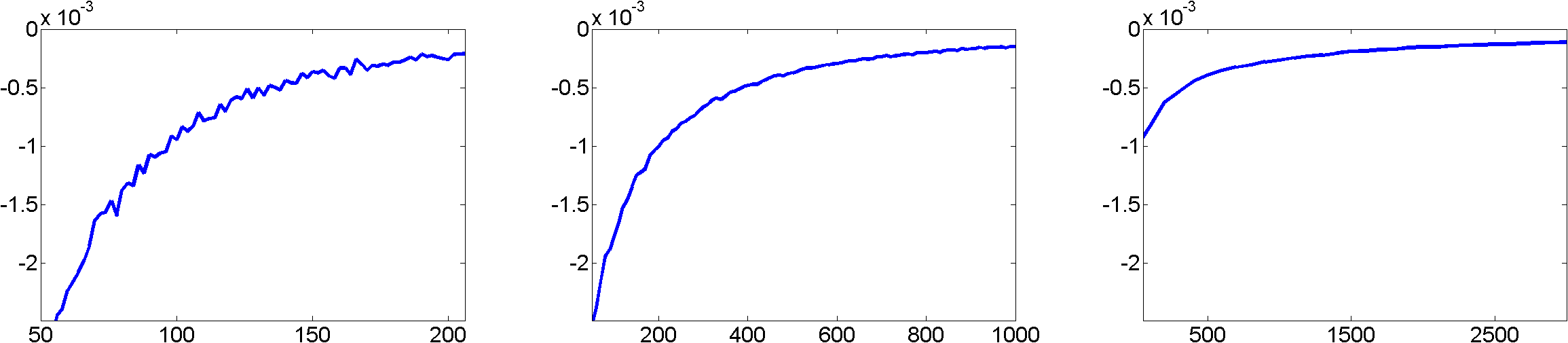}
\end{center}
\caption{The value of $y$ plotted against size of dataset for clique containing a) three variables b) four variables c) five variables.}
\label{dataSize}
\end{figure}

Next we will conduct searches for context-specific independencies in real data, we investigate two particular datasets that have been considered in multiple works concerning learning of the model structure. The first dataset includes prognostic factors for coronary heart disease and can be found in \citet{Edwards85}. The dataset contains 1841 observations on the six variables described in Table \ref{heart_var}.
\begin{table}[tbh]
\begin{center}
\begin{tabular}
[c]{cll} \hline
Variable & Meaning & Range \\ \hline
$X_1$ & Smoking & No = 0, Yes = 1 \\
$X_2$ & Strenuous mental work & No = 0, Yes = 1 \\
$X_3$ & Strenuous physical work & No = 0, Yes = 1 \\
$X_4$ & Systolic blood pressure $> 140$ & No = 0, Yes = 1 \\
$X_5$ & Ratio of beta and alpha lipoproteins $> 3$ & No = 0, Yes = 1 \\
$X_6$ & Family anamnesis of coronary heart disease & No = 0, Yes = 1 \\ \hline
\end{tabular}
\end{center}
\caption{Variables in coronary heart disease data.}
\label{heart_var}
\end{table}
Using the same setup as for the synthetic data, the two best decomposable SGs are depicted in Figure \ref{heart}. They have the log-unnormalized posterior values of $-6715.90$ and $-6716.66$, respectively. The underlying graph in the optimal decomposable SG coincides with the optimal undirected graph as found by \citet{Corander08} and is one of the two graphs suggested by \citet{Edwards85}. The discussion in \citet{Whittaker90} also suggests the possible inclusion of the edges $\{2, 5\}$ and $\{1, 2\}$. Compared to these sources our models are highly similar. However, in addition to the global independencies, our framework suggests for instance the context specific independencies $X_1 \perp X_4 \mid X_5=1$ and $X_4 \perp X_5 \mid X_1=0$.

The interpretation of this is that the knowledge that a person smokes and has a ratio of beta and alpha lipoproteins less than or equal to $3$ will affect the systolic blood pressure in one way and all the other variations for smoking and ratio of beta and alpha lipoproteins in another way. A simplified version would be to say that given that a person has a ratio of beta and alpha lipoproteins larger than $3$, whether or not he smokes or not is unlikely to affect his systolic blood pressure. Interestingly the labeled edges are those that some sources suggest should be included in the model whilst other sources omit from the model.

\begin{figure}[htb]
\begin{center}
\includegraphics{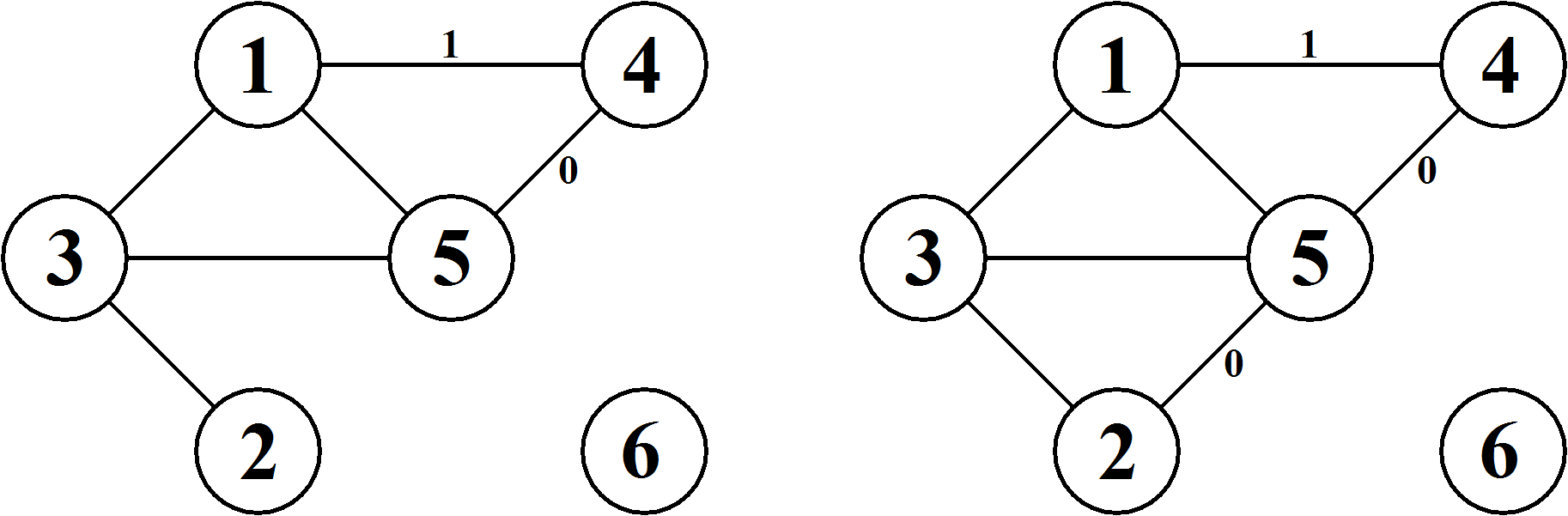}
\end{center}
\caption{SGs with highest posterior probabilities for the heart disease data.}
\label{heart}
\end{figure}

Next we consider a more comprehensive dataset involving 25 variables. This dataset is derived from the answers given by 1806 candidates in the Finnish parliament elections of 2011, in a questionnaire issued by the newspaper Helsingin Sanomat \citep{HelsinginSanomat11}. The questionnaire contains a total of 30 questions, of these 25 are on a ordinal scale. The answers given to these 25 questions by the candidates are transformed to the binary variables listed in Appendix C.

The SG with highest posterior probability is shown in Figure \ref{hs}. The labels are not explicitly given in the graph, due to limited space, instead all the labeled edges are colored red. This maximal regular SG contains 72 edges of which 36 are labeled. The graph contains a total of 87 labels and has a log-unnormalized posterior value of -21949.13. Conducting a search for the best ordinary graphical model results in a graph with 70 edges and a log-unnormalized posterior value of -22043.15. These two graphs share 62 edges, implying that the induced dependence structures resemble each other to a considerable degree.

\begin{figure}[htb]
\begin{center}
\includegraphics{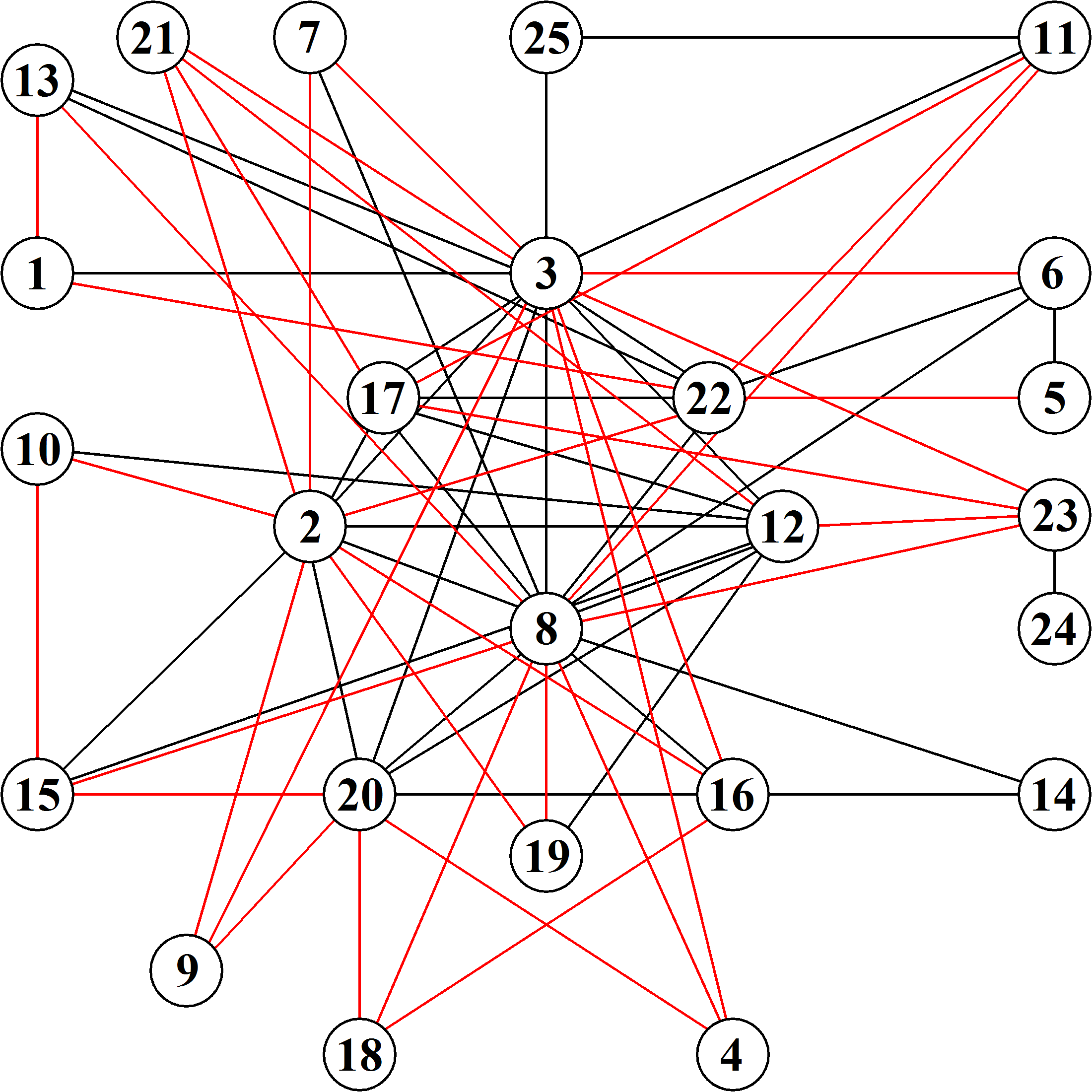}
\end{center}
\caption{Optimal SG for the parliament election data, labeled edges are colored red.}
\label{hs}
\end{figure}

Figure \ref{hsCliques} displays two cliques, found in both the optimal SG and optimal ordinary graph, with the labels associated with the SG. The SG in Figure \ref{hsCliques}a induces a fairly straightforward context-specific dependence structure. Given that we know a candidate's opinion on mandatory military service (variable 12), knowing that the candidate is against equal rights for homosexuals to adopt children (variable 2) is likely to have the same affect on a candidate's view on singing Christian hymns in school (variable 19) as knowing that the candidate is in favor of economic help packages for struggling Euro countries (variable 8). The context-specific dependence structure induced by the SG in Figure \ref{hsCliques}b is much more intricate. However, a simple fact is that a probability distribution faithful to this component of the SG includes 21 free parameters, whereas the corresponding ordinary clique would include 31 free parameters. In total, the optimal SG in Figure \ref{hs} induces a distribution with 324 free parameters while the underlying graph and optimal ordinary graph induce distributions with 407 and 368 free parameters, respectively. This means that the SG induces a more elaborate dependence structure using a substantially smaller number of parameters.

\begin{figure}[htb]
\begin{center}
\includegraphics{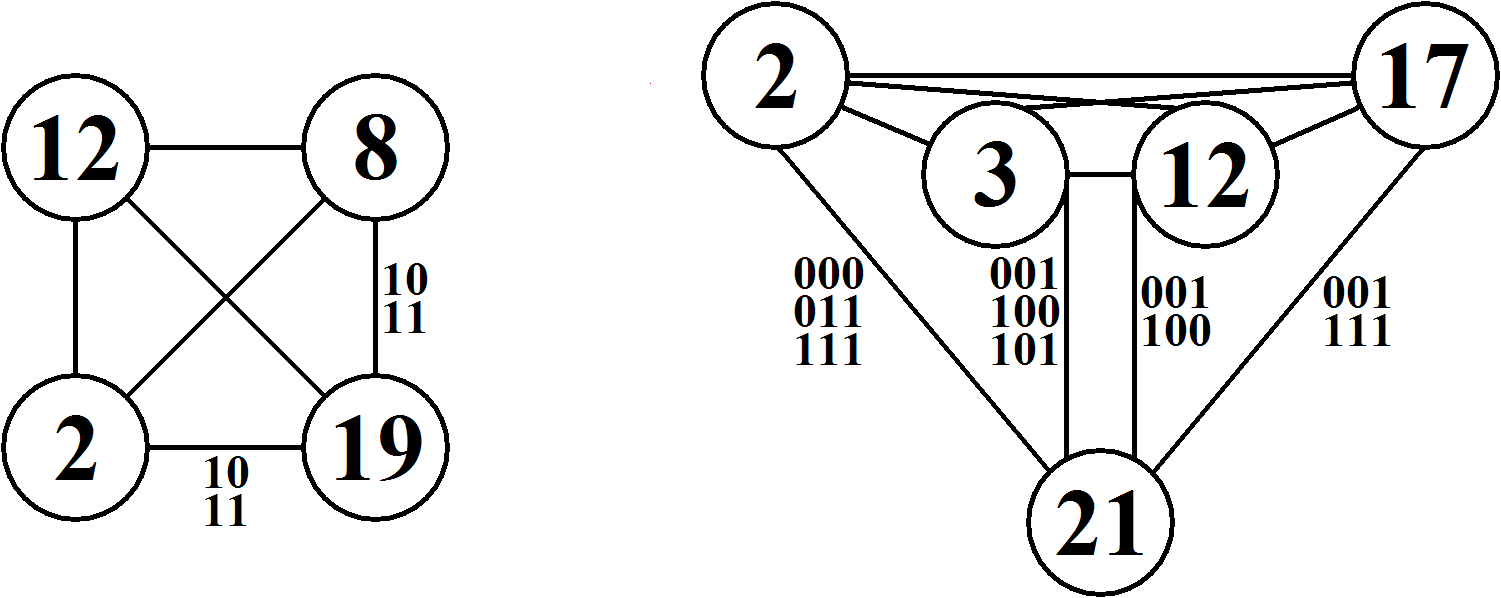}
\end{center}
\caption{Two cliques found both in the optimal SG and optimal ordinary graph.}
\label{hsCliques}
\end{figure}

\section{Discussion}
The versatility of probabilistic graphical models has become clear through their popularity over a wide variety of application areas. On the other hand, their fairly restrictive global form of dependence structures has inspired the development of many generalizations of graphical models where independence can be a function of a context in the outcome space. In fact, similar developments have taken place for the class of Markov chain models, where the variable-order and variable-length Markov chains aim at a generalization of ordinary higher-order Markov chains where the dependence on the history of the process is context-specific \citep{Rissanen83, Weinberger95, Buhlmann99, Bacallado11}. Our formulation of the simultaneous context-specific independence restrictions allows for the derivation of an analytical Bayesian scoring function for decomposable SGs, which is particularly useful for fast learning purposes and leads to a more expressive model class than those considered by \citet{Boutilier96}, \citet{Friedman96}, \citet{Corander03a}, and \citet{Koller09}. In the future it would be fruitful to develop inference methods also for non-decomposable SGs, which do not enjoy an analytically tractable expression for the marginal likelihood. This would further reduce the constraints imposed on the dependence structure and allow for even more expressive range of context-specific independencies to be explored.

\section*{Acknowledgments}
H.N. and J.P. were supported by the Foundation of \AA bo Akademi University, as part of the grant for the Center of Excellence in Optimization and Systems Engineering. J.C. was supported by the ERC grant no. 239784 and academy of Finland grant no. 251170. T.K. was supported by a grant from the Swedish research council VR/NT.

\bibliographystyle{henrik}
\bibliography{biblio}

\section*{Appendix A}
Proof of Theorem \ref{minimal}. \newline\newline First, we prove that a joint distribution, over the variables $X_{\Delta} = (X_{1}, \ldots, X_{d})$, faithful to a decomposable SG, with $G$ as the underlying graph and $L$ as the set of strata, can always be factorized minimally. Since $G$ is a decomposable undirected graph, $P_{\Delta}$ can be factorized \citep{Lauritzen96} according to
\[
P_{\Delta}(X_{\Delta})=\frac{\prod_{C \in \mathcal{C}(G)}P_{C}(X_{C})}{\prod_{S \in \mathcal{S}(G)}P_{S}(X_{S})}.
\]
For any separator $S \in\mathcal{S}(G)$ it holds that any variable $X_{\delta} \in X_S$ is dependent on each of the other variables in the set $X_S \backslash X_{\delta}$, irrespectively of any context-specific independencies, since decomposable SGs cannot have any labeled edges within separators. This implies that any factorization of the joint distribution $P_{S}(X_{S})$ using a sequence of conditional probabilities  $P(X_{1}) P(X_{2} \mid X_{1}) \ldots P(X_{d_{S}} \mid X_{1},\ldots, X_{d_{S} - 1})$ with $d_S = |S|$ will be void of false dependencies. In a clique without labeled edges the joint distribution can be minimally factorized in the same manner as for a separator.  

Consider the clique distribution $P_{C}(X_{C})$, $C \in\mathcal{C}(G)$, containing the variables $(X_{1}, \ldots, X_{d_{C}})$ with $d_{C} = |C|$. Assume that, in accordance with the definition of decomposable SGs, all the labeled edges have at least one node in common. The corresponding variable is chosen to be the last variable in the ordering, i.e. the variable $X_{d_C}$. In the case where the clique only contains one labeled edge $\{\delta, \gamma\}$ the choice of $X_{d_C}$ is ambiguous, either we choose $X_{d_C} = X_{\delta}$ or $X_{d_C} = X_{\gamma}$. It now follows that the factorization $P(X_{1}, \ldots, X_{d_{C} - 1}) = P(X_{1}) P(X_{2} \mid X_{1}) \ldots P(X_{d_{C} - 1} \mid X_{1},\ldots, X_{d_{C} - 2})$ contains no false dependencies. This can be seen from the stratified graph as all pairs of nodes corresponding to the variables in the set $(X_{1}, \ldots, X_{d_{C}-1})$ will be connected by an unlabeled edge. The final situation to investigate is the conditional probability $P(X_{d_{C}} \mid X_{1}, \ldots, X_{d_{C} - 1})$. This factor could potentially contain false dependencies as the edges leading to the node corresponding to $X_{d_{C}}$ are allowed to be labeled. However, since the values of $(X_{1}, \ldots, X_{d_{C} - 1})$ can be considered known as they appear earlier in the factorization of the joint distribution, and as these are the variables that determine whether or not the conditions of the labels in question are satisfied, it is known which dependencies can be excluded from $P(X_{d_{C}} \mid X_{1}, \ldots, X_{d_{C} - 1})$. Hence, it is always possible to avoid introducing false dependencies for such a clique. This proves that a minimal factorization always exists for a joint distribution faithful to a decomposable SG.

We now prove that a joint distribution, that is faithful to an SG and can be factorized minimally, is necessarily faithful to a decomposable SG. Firstly, for any non-decomposable graph $G$, the associated joint distribution cannot be factorized minimally which is shown as follows. Here a node $\delta$ is denoted $V_{\delta}$. A non-decomposable undirected graph contains a path of nodes $(V_1, \ldots, V_k, V_1)$ with $k \geq 4$, see for instance \citet[p.~127]{Koski09}, such that of all the nodes in the path, node $V_i$ is only adjacent to the nodes $V_{i-1}$ and $V_{i+1}$. To begin the factorization process, one needs to choose an initial variable in the sequence, say $X_1$. When the next factor is chosen, we are constrained to choosing either $P(X_2 \mid X_1)$ or $P(X_k \mid X_1)$ as any other variable $X_j$ is independent of $X_1$ given $X_{\Delta} \backslash (X_1 \cup X_j)$, which implies that such a choice would introduce a false dependency. If $P(X_2 \mid X_1)$ is chosen as the next factor we obtain $P(X_{1},X_{2}) = P(X_{1}) P(X_{2} \mid X_{1})$. After that it must be decided which conditional probability to add next into the factorization. Our options are, following the reasoning above, $P(X_k \mid X_1,X_2)$ and $P(X_3 \mid X_1,X_2)$. However, both of these options introduce a false dependency since for instance, $X_1 \perp X_3  \mid  X_{\Delta} \backslash (X_1 \cup X_3)$. Adding the factor $P(X_3 \mid X_2)$ is not a permissible option either, since $X_1$ and $X_3$ are still dependent at this stage of the factorization process. This ultimately means that no non-decomposable graph can be associated with a joint distribution possessing a minimal factorization, thus any SG possessing such a property must have a decomposable underlying graph.

Assume now that a label would be allowed on an edge in a separator $S$ between two arbitrary cliques $C_1$ and $C_2$. To be able to deduce whether or not the context of the label is satisfied we would need to know the values of $(X_{C_1} \backslash X_S)$ and $(X_{C_2} \backslash X_S)$. But as long as the values of the variables in the separator $X_S$ are unknown, the variables in $(X_{C_1} \backslash X_S)$ and $(X_{C_2} \backslash X_S)$ remain dependent. This in turn renders a minimal factorization impossible, and hence a minimal factorization can only exist if no labels are assigned to edges in separators.

Finally, to prove that a minimal factorization requires that each labeled edge in a clique must have at least one node in common, consider a situation where this assumption does not hold. Let $\{1, 2\}$, and $\{3, 4\}$ denote two labeled edges in a clique and assume that the joint distribution over the variables $(X_1, X_2, X_3, X_4)$ can be minimally factorized. However, to ascertain whether a context-specific independence is present between $X_1$ and $X_2$, it is necessary to know the values of $X_3$ and $X_4$. Consequently, since the context-specific dependence structure between these two variables is unknown until the values of $X_1$ and $X_2$ are fixed, it becomes impossible to guarantee a minimal factorization. The other representative possibility is that we would instead have three labeled edges $\{1, 2\}$, $\{2, 3\}$, and $\{1, 3\}$, where all pairs of edges share a single node but this is not the same node for all of them. Under such circumstances a minimal factorization is again impossible, because when considering the factor $P(X_2 \mid X_1)$, $X_2$ may or may not be dependent on $X_1$ given the value of $X_3$. Thus, for a minimal factorization to be possible all labeled edges in a clique must have at least one node in common.

\section*{Appendix B}
The following Matlab code is used to generate a dataset with a dependence structure following that of the stratified graph in Figure \ref{synthetic}a.

\begin{verbatim}
function data=dataSim(n)
  x2=rand(n,1)<0.5;
  x3=zeros(n,1);
  bol2=x2==1;
  x3(bol2)=rand(sum(bol2),1)<0.8;
  x3(~bol2)=rand(sum(~bol2),1)<0.3;
  bol3=x3==1;
  x4=zeros(n,1);
  x4(bol2 & bol3)=rand(sum(bol2 & bol3),1)<0.6;
  x4(bol2 & ~bol3)=rand(sum(bol2 & ~bol3),1)<0.8;
  x4(~bol2 & bol3)=rand(sum(~bol2 & bol3),1)<0.2;
  x4(~bol2 & ~bol3)=rand(sum(~bol2 & ~bol3),1)<0.4;
  bol4=x4==1;
  x1=zeros(n,1);
  x1(~bol2 & ~bol3 & ~bol4)=rand(sum(~bol2 & ~bol3 & ~bol4),1)<0.5;
  x1(~bol2 & ~bol3 & bol4)=rand(sum(~bol2 & ~bol3 & bol4),1)<0.7;
  x1(~bol2 & bol3 & ~bol4)=rand(sum(~bol2 & bol3 & ~bol4),1)<0.2;
  x1(~bol2 & bol3 & bol4)=rand(sum(~bol2 & bol3 & bol4),1)<0.4;
  x1(bol2 & ~bol3 & ~bol4)=rand(sum(bol2 & ~bol3 & ~bol4),1)<0.2;
  x1(bol2 & ~bol3 & bol4)=rand(sum(bol2 & ~bol3 & bol4),1)<0.8;
  x1(bol2 & bol3 & ~bol4)=rand(sum(bol2 & bol3 & ~bol4),1)<0.2;
  x1(bol2 & bol3 & bol4)=rand(sum(bol2 & bol3 & bol4),1)<0.8;
  x5=zeros(n,1);
  x5(~bol3)=rand(sum(~bol3),1)<0.8;
  x5(bol3 & ~bol4)=rand(sum(bol3 & ~bol4),1)<0.2;
  x5(bol3 & bol4)=rand(sum(bol3 & bol4),1)<0.5;
  bol5=x5==1;
  x6=zeros(n,1);
  x6(~bol5)=rand(sum(~bol5),1)<0.3;
  x6(bol5)=rand(sum(bol5),1)<0.7;
  bol6=x6==1;
  x7=rand(n,1);
  x7(bol5 | ~bol6)=rand(sum(bol5 | ~bol6),1)<0.4;
  x7(~bol5 & bol6)=rand(sum(~bol5 & bol6),1)<0.8;
  data=double([x1 x2 x3 x4 x5 x6 x7]);
end
\end{verbatim}

\section*{Appendix C}
The following is a list of the questions presented to the candidates of the Finnish parliament elections of 2011 by the newspaper Helsingin Sanomat.

\begin{enumerate}
\item Since the mid-1990's the income differences have grown rapidly in Finland. How should we react to this? \\
0 - The income differences do not need to be narrowed. \\
1 - The income differences need to be narrowed.

\item Should homosexual couples have the same rights to adopt children as heterosexual couples? \\
0 - Yes. \\
1 - No.

\item In 2010 the Finnish government issued permits for two new nuclear power plants. The energy company Fortum was not issued a permit. Should Fortum also be issued a permit? \\
0 - Yes. \\
1 - No.

\item Child benefits are paid for each child under the age of 18 living in Finland, independent of the parents' income. What should be done about child benefits? \\
0 - The income of the parents should not affect the child benefits. \\
1 - Child benefits should be dependent on parents' income.

\item Should senior citizens by law be guaranteed a subjective right to medical treatment by the municipalities. \\
0 - Yes. \\
1 - No.

\item The average age when people start to receive old age pension is 63-68 years. In my opinion the lower limit for old age pension should be \\
0 - raised. \\
1 - kept as is or lowered.

\item Since the 1990's a folded index, depending to 80 percent on consumer prices and 20 percent on wage development, has been used to determine the increase in pensions. Because of this the increase in pensions has been less than the general wage increase. What should be done about the folded index? \\
0 - The folded index should to more than 50 percent depend on consumer prices. \\
1 - The folded index should to 50 percent or less depend on consumer prices.

\item Finland, along with other Euro countries, have financed help packages for Euro countries suffering severe economic problems. What is your opinion on the help packages? \\
0 - Paying for the help packages was a mistake. \\
1 - Paying for the help packages benefits Finland in the long run.

\item The European Commission has suggested a world-wide transaction tax, which would introduce a taxation on a number of financial transactions. What is your opinion regarding the transaction tax? \\
0 - A transaction tax should not be introduced. \\
1 - A transaction tax should be introduced.

\item What should be done about the tax deductibility of mortgage interests? \\
0 - The tax deductibility should be expanded or kept unchanged. \\
1 - The tax deductibility should be reduced.

\item A property tax is paid for developed and undeveloped land. Should this tax be expanded to include wood- and farmland? \\
0 - Yes. \\
1 - No.

\item In Finland military service is mandatory for all men. What is your opinion on this? \\
0 - The current practice should be kept or expanded to also include women. \\
1 - The military service should be more selective or abandoned altogether.

\item Should Finland apply for NATO membership? \\
0 - Yes. \\
1 - No.

\item Should Finland in its affairs with China and Russia more actively debate issues regarding human rights and the state of democracy in these countries? \\
0 - Yes. \\
1 - No.

\item Russia has prohibited foreigners from owning land close to the borders. In recent years, Russians have bought thousands of properties in Finland. How should Finland react to this? \\
0 - Finland should not restrict foreigners from buying property in Finland. \\
1 - Finland should restrict foreigners rights to by property and land in Finland.

\item Finland has pledged to follow the UN set target that in 2015, 0.7 percent of the BNP should be used for foreign aid. In 2010 Finland gave 0.55 percent of its BNP in foreign aid. What is your opinion on foreign aid? \\
0 - Foreign aid should not be increased. \\
1 - Foreign aid should be increased.

\item How should handgun restriction laws be amended in Finland? \\
0 - The restriction laws regarding handguns should be expanded. \\
1 - The restriction laws regarding handguns should be kept as is or reduced.

\item Two foreign languages, of which one is the second native language (Finnish or Swedish), are compulsory in the Finnish primary school. Should study of the second native language be made voluntary? \\
0 - Yes. \\
1 - No.

\item Is it appropriate to sing 'Suvivirsi', a hymn traditionally sung in schools before the start of the summer holidays, in public schools? \\
0 - Yes. \\
1 - No.

\item Recently, several laws have been passed making Finnish immigration laws more strict. In your opinion, what is the state of the current immigration laws? \\
0 - To lenient. \\
1 - Appropriate or to strict.

\item The Saimaa ringed seal is classified as a very endangered specie. Are the current protection efforts sufficient? \\
0 - Yes. \\
1 - No.

\item During recent years municipalities have outsourced many services to privately owned companies. What is your opinion on this? \\
0 - Outsourcing should be used to an even higher extent. \\
1 - Outsourcing should be limited to the current extent or decreased.

\item The current number of municipalities in Finland is 336. In your opinion what is the optimal number of municipalities in Finland? \\
0 - Less than 290. \\
1 - More than 290.

\item The creation of a metropolis around Helsinki has been lively debated. What should the government do in this case? \\
0 - The government can at most try to encourage a closer collaboration of the involved municipalities. \\
1 - The government must create a metropolis, using force if necessary.

\item Currently, a system is in place where tax income from more wealthy municipalities is transferred to to less wealthy municipalities. In practice this means that municipalities in the Helsinki region transfers money to the other parts of the country. what is your opinion of this system? \\
0 - The current system is good, or even more money should be transferred. \\
1 - The Helsinki region should be allowed to keep more of its tax income.

\end{enumerate}

\end{document}